\documentclass[letterpaper]{article} 
\usepackage{aaai2026}  
\usepackage{times}  
\usepackage{helvet}  
\usepackage{courier}  
\usepackage[hyphens]{url}  
\usepackage{graphicx} 
\usepackage{graphicx} 
\usepackage{natbib}  
\usepackage{caption} 
\usepackage{algorithm}
\usepackage{algorithmic}
\usepackage{amsmath,amssymb,amsfonts}
\usepackage{algorithmic}
\usepackage{textcomp}
\usepackage{xcolor}
\usepackage{amsthm}
\usepackage{pdfpages}
\usepackage{subcaption}
\usepackage{booktabs}

\newtheorem{theorem}{Theorem}

\newcommand{\Retrace}{ReTrace}

\usepackage{newfloat}
\usepackage{listings}
\DeclareCaptionStyle{ruled}{labelfont=normalfont,labelsep=colon,strut=off} 
\floatstyle{ruled}
\newfloat{listing}{tb}{lst}{}
\floatname{listing}{Listing}
\title{Robust SDE Parameter Estimation Under Missing Time Information Setting}
\author{
    Long Van Tran,
    Truyen Tran,
    Phuoc Nguyen
}
\affiliations {
    Applied Artificial Intelligence Initiative (A\textsuperscript{2}I\textsuperscript{2}), Deakin University, Geelong, Australia\\
    \{s224930257, truyen.tran, phuoc.nguyen\}@deakin.edu.au
}

\begin{document}

\maketitle

\begin{abstract}
Recent advances in stochastic differential equations (SDEs) have enabled robust modeling of real-world dynamical processes across diverse domains, such as finance, health, and systems biology. However, parameter estimation for SDEs typically relies on accurately timestamped observational sequences. When temporal ordering information is corrupted, missing, or deliberately hidden (e.g., for privacy), existing estimation methods often fail. In this paper, we investigate the conditions under which temporal order can be recovered and introduce a novel framework that simultaneously reconstructs temporal information and estimates SDE parameters. Our approach exploits asymmetries between forward and backward processes, deriving a score-matching criterion to infer the correct temporal order between pairs of observations. We then recover the total order via a sorting procedure and estimate SDE parameters from the reconstructed sequence using maximum likelihood. Finally, we conduct extensive experiments on synthetic and real-world datasets to demonstrate the effectiveness of our method, extending parameter estimation to settings with missing temporal order and broadening applicability in sensitive domains.
\end{abstract}



\section{Introduction}
Stochastic differential equations (SDEs) are fundamental tools for modeling dynamical systems under uncertainty across diverse domains such as finance, health, and biology~\cite{pavliotis2014stochastic, qian2021integrating, kacprzyk2024odediscovery}. The Ornstein--Uhlenbeck (OU) process, a linear Gaussian SDE, is particularly notable for its analytical tractability and broad applicability~\cite{gardiner2004handbook}. 

\begin{figure}[ht]
    \centering
    \includegraphics[width=230pt]{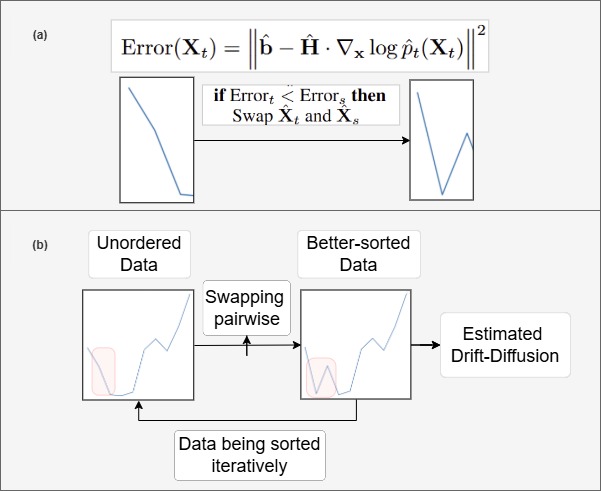}
    \caption{Our sorting procedure leverages drift-score discrepancy to reorder data. (a) Compare errors for each states pair in the Data Reordering stage. (b) Alternating between sorting data and estimating parameters.}\label{method_overview}
\end{figure}

Traditional parameter estimation techniques for SDEs, e.g. maximum likelihood estimation (MLE), crucially depend on access to temporally ordered data~\cite{sarkka2019applied, guan2024identifyingdriftdiffusioncausal}. However, in many real-world scenarios, the temporal order is corrupted, incomplete, or intentionally removed. This challenge is especially crucial in domains such as healthcare and finance, where privacy concerns, data anonymization, or system errors can lead to datasets with missing or inaccurate timestamps. In such cases, the loss of temporal information renders standard inference methods unreliable, necessitating new approaches that can recover both the time order and the underlying SDE parameters.

Existing work primarily focused on trajectory inference from population snapshots to reconstruct latent dynamics when time is observed but individual trajectories are untracked~\cite{guan2024identifyingdriftdiffusioncausal, gu2025partially}. However, these approaches assume at least partial temporal information is available and do not address settings where the timestamp of each observation is unknown or incorrect. As a result, when observations are unordered, neither of these methods guarantee accurate parameter estimation or trajectory reconstruction.

This issue is particularly prominent in the analysis of electronic medical records (EMR), where timestamps are often subject to recording delays, asynchronous updates, and missingness~\cite{nguyen2016deepr, pham2017predicting}. Thereby, patient histories extracted from EMR commonly contain irregular, noisy, or misordered event sequences~\cite{johnson2016mimic}. Inference of interventional SDEs from such data becomes fundamentally non-identifiable: causal effect estimates or post-intervention predictions may be reversed, potentially leading to invalid or even harmful conclusions, e.g., wrongly inferring that a treatment is detrimental when it is beneficial.

In this work, we address the challenge of \emph{jointly recovering temporal order and estimating SDE parameters}  from unordered observations. We analyze the identifiability of time direction in time-homogeneous linear SDEs, and propose a new framework (Fig. \ref{method_overview}) that exploits statistical asymmetries between forward and backward dynamics via a pairwise score-matching criterion. From this, we aggregate the local temporal direction to reconstruct the global order via a sorting algorithm. We then apply maximum likelihood estimation to the reconstructed data for parameter inference. We validated our framework on both synthetic and real-world datasets and showed reliable temporal order recovery and robust parameter estimation, even under severe order corruption.

\paragraph{Our contributions are:}
\begin{itemize}
\item We provide theoretical conditions under which temporal order is identifiable from unordered observations of linear SDEs.
\item We introduce a novel pairwise score-matching criterion to infer the directionality between observation pairs.
\item We develop a practical algorithm that reconstructs global temporal order, enabling classical parameter estimation techniques to be applied to unordered data.
\item We demonstrate, through extensive experiments on synthetic and real-world datasets, the robustness of our method to severe temporal order corruption.
\end{itemize}

\section{Background}
In this section, we briefly review stochastic differential equations (SDEs), relevant applications, existing methods for inference, and highlight non-identifiability issues when temporal order is missing.

\paragraph{Stochastic Differential Equations (SDEs)}
model the temporal evolution of systems influenced by both deterministic dynamics (drift) and random perturbations (diffusion). A general $d$-dimensional, time-homogeneous SDE driven by an $m$-dimensional Brownian motion $\mathbf{W}_t$ is written as
\begin{equation}
    d\mathbf{X}_t = \mathbf{b}(\mathbf{X}_t)\,dt + \sigma(\mathbf{X}_t)\,d\mathbf{W}_t, \quad \mathbf{X}_0 \sim p_0,
\end{equation}
where $\mathbf{b}: \mathbb{R}^d \to \mathbb{R}^d$ is the drift vector field, and $\sigma: \mathbb{R}^d \to \mathbb{R}^{d \times m}$ is the diffusion coefficient. Under standard Lipschitz and growth conditions, the SDE admits a unique strong solution \cite{10.1214/19-AAP1507}. SDEs play a foundational role in diverse domains, including modeling hydrological flows \cite{BEVEN199319, doi:10.1061/(ASCE)0733-9399(2007)133:4(422)}, financial assets \cite{0b9b8115-a8b8-3422-8e1c-a62077de6621}, and as the backbone of modern diffusion-based generative models \cite{sohl2015deep, ho2020denoising, song2020score, meng2022sdeditguidedimagesynthesis, nguyen2025hedit}.

This work focuses on the subclass of time-homogeneous SDEs with additive, constant (possibly anisotropic) diffusion:
\begin{equation}\label{eq:additive_sde}
    d\mathbf{X}_t = \mathbf{b}(\mathbf{X}_t)\,dt + \mathbf{G}\,d\mathbf{W}_t, \quad \mathbf{X}_0 \sim p_0,
\end{equation}
where $\mathbf{b}$ is the drift function as above, $\mathbf{W}_t$ is an $m$-dimensional standard Wiener process, $\mathbf{G} \in \mathbb{R}^{d \times m}$ is a constant matrix and we also note that the \emph{observational diffusion matrix} is $\mathbf{H} = \mathbf{G}\mathbf{G}^\top \succ 0$ \cite{NEURIPS2023_ca642f8e}. Typically, we can only estimate $\mathbf{H}$ from data (rather than $\mathbf{G}$). A particularly important special case is the time-homogeneous linear additive noise SDE,
\begin{equation}\label{time_homo_linear_additive_noise_SDE}
    d\mathbf{X}_t = \mathbf{A}\mathbf{X}_t\,dt + \mathbf{G}\,d\mathbf{W}_t, \quad \mathbf{X}_0 \sim p_0,
\end{equation}
which generalizes the Ornstein–Uhlenbeck process and is a fundamental model for both physical and statistical applications.

Classical methods for SDE parameter estimation, such as maximum likelihood estimation and Kalman filtering, assume temporally ordered data and are widely used in statistics, finance, and biology~\cite{PhysRevE.54.2084, NEURIPS2023_ca642f8e, bao2024applicationkalmanfilterstochastic, carter2024parameterestimationornsteinuhlenbeckprocess, guan2024identifyingdriftdiffusioncausal, sun2025robustparameterestimationdynamical, gu2025partially}. These approaches exploit the Markov property and transition densities of SDEs, and can be extended to irregularly sampled or partially observed trajectories using EM algorithms or state-space models. However, these works fundamentally rely on accurate or partially known time information.

\paragraph{Related Work} To recover temporal structure in unordered or noisy data, graph-based approaches such as the Minimum Spanning Tree (MST)~\cite{trapnell2014dynamics} and pseudotime methods including Diffusion Pseudotime (DPT) and related latent time inference techniques~\cite{Haghverdi2016, doi:10.1073/pnas.1408993111} have been proposed in the context of single-cell genomics. While these algorithms reconstruct plausible orderings via pairwise distances, depth-first seach, or nearest neighbor-based transition matrix, they do not use SDE transition dynamics and provide no guarantees for parameter identifiability or causal inference. We will compare these methods with ours in the experiment section.

\section{Methods}
We assume the process $\{\mathbf{X}_t\}_{t \geq 0}$ evolves in $\mathbb{R}^d$ according to a time-homogeneous SDE with general drift and constant additive noise as in Eq.~\eqref{eq:additive_sde}. We observe $N$ independent trajectories $\bigl\{ \mathbf{X}^{(j)}_{i\Delta t} \bigr\}_{i=0}^T$ on an equally spaced time grid $t_i = i\Delta t$, where $\Delta t > 0$ is the step size, and $T$ is the number of time steps. For each trajectory, the increments are given by
\[
    \Delta\mathbf{X}_i^{(j)} = \mathbf{X}_{i+1}^{(j)} - \mathbf{X}_i^{(j)}, \quad j = 1, \ldots, N, \; i = 0, \ldots, T-1.
\]

\subsection{Problem Formulation}
Now suppose those observed trajectories are recorded on an equally spaced time grid but at \textbf{unknown time points}, i.e., the original temporal order of the data, which follows the assumed format in \eqref{time_homo_linear_additive_noise_SDE}, has been fully permuted. That is, the observed dataset is a permutation of the true trajectory:
\begin{equation}
    \mathcal{X} = \bigl\{\mathbf{X}_{\pi(1)}, \mathbf{X}_{\pi(2)}, ..., \mathbf{X}_{\pi(n)}\bigr\},
\end{equation}
where $\pi: \bigl\{1, ..., n\bigr\} \rightarrow \bigl\{1, ..., n\bigr\}$ is an unknown permutation representing the lost temporal order. Then our problem can be formally stated as:
\begin{equation}
\begin{aligned}
    \text{Given: } \mathcal{X} &= \bigl\{\mathbf{X}_{\pi(1)}, \mathbf{X}_{\pi(2)}, ..., \mathbf{X}_{\pi(n)}\bigr\} \\
    \text{Estimate: } \mathbf{\theta} &= (\mathbf{A}, \mathbf{H}), \text{ Recover: } \mathcal{X} \text{ (correct order)}
\end{aligned}
\end{equation}
such that our model (with parameters estimated) is most consistent with the correctly reconstructed data.

\subsection{Conditions for Identifiability of Time Direction}
When temporal order is missing, observed trajectories are not aligned with the true time axis. We first characterize the conditions under which the underlying time direction is statistically identifiable or not. Based on these results, we develop a score-based iterative method to recover the correct order and estimate the SDE parameters.

\begin{theorem}[Non-identifiability of Time Direction]\label{thm:nonidentify}
Suppose the drift satisfies $\mathbf{b}(\mathbf{x}) = \frac{1}{2} \mathbf{H} \nabla \log p_*(\mathbf{x})$, i.e., the process is reversible and satisfies detailed balance with respect to $p_*(\mathbf{x})$. Then, for any finite sequence of observations $(\mathbf{x}_0, \ldots, \mathbf{x}_n)$ drawn from the stationary process, the joint distribution is symmetric under time reversal:
\begin{equation}
    p(\mathbf{x}_0, \mathbf{x}_1, \ldots, \mathbf{x}_n) = p(\mathbf{x}_n, \mathbf{x}_{n-1}, \ldots, \mathbf{x}_0).
\end{equation}
Thus, the direction of time cannot be statistically identified from such data, by any criterion based on observed paths.
\end{theorem}

\begin{proof}\emph{(Sketch of proof).}
Detailed balance implies the stationary probability current vanishes and the transition kernel satisfies $p_\tau(\mathbf{y}|\mathbf{x})p_*(\mathbf{x}) = p_\tau(\mathbf{x}|\mathbf{y})p_*(\mathbf{y})$. By induction, the joint law of a stationary path is symmetric under reversal: $p(\mathbf{x}_0,\ldots,\mathbf{x}_n) = p(\mathbf{x}_n,\ldots,\mathbf{x}_0)$. Thus, no statistical test can recover the time direction.
\end{proof}

The stationary probability current, defined as
\begin{equation}
    \mathbf{J}(\mathbf{x}) = \mathbf{b}(\mathbf{x})\,p_*(\mathbf{x}) - \frac{1}{2}\mathbf{H}\nabla p_*(\mathbf{x}),
    \label{eq:j}
\end{equation}
serves as a key criterion for distinguishing reversible and irreversible diffusion processes. A process is reversible if and only if $\mathbf{J}(\mathbf{x}) = 0$ for all $\mathbf{x}$; in this case, the system satisfies detailed balance, and the stationary joint law is symmetric under time reversal. This symmetry implies that no statistical procedure can identify the true direction of time from observed data alone. In contrast, when $\mathbf{J}(\mathbf{x}) \neq 0$ for some $\mathbf{x}$ (i.e., the process is irreversible and detailed balance is broken), the stationary distribution is accompanied by a nonzero probability current, and the true time direction becomes statistically identifiable from data. Thus, the irreversibility of the underlying SDE fundamentally determines whether temporal order can be inferred from sample paths. When the process in Eq. \eqref{time_homo_linear_additive_noise_SDE} is irreversible, the drift $\mathbf{b}$ is not a gradient field and the forward and backward SDEs differ by a term involving the score (gradient of the log density) of the marginal distribution:
\begin{equation}
\bar{\mathbf{b}}(\mathbf{x}) = \mathbf{b}(\mathbf{x}) - \mathbf{H}\nabla{\mathbf{x}}\log p_t(\mathbf{x}),
\end{equation}
where $p_t(\mathbf{x})$ is the marginal density at time $t$. This inherent asymmetry allows for a statistical test of the time direction: the correct temporal order is the one in which the empirical drift increments most closely fit the predicted SDE dynamics, once the score correction is accounted for. Thus, we can infer the correct temporal order from an unordered sequence.

Empirically, given two candidate orderings of adjacent snapshots $\mathbf{X}_a$ and $\mathbf{X}_b$ (separated by presumed time interval $\Delta t$), we can estimate the empirical drift, compute the empirical score (using sample mean and covariance), and quantify the fit error between the observed increment and the score-corrected drift prediction. The time direction that yields the lower total error across all pairs is thus favored. This leads directly to our next result:
\begin{theorem}[Identifiability of Time Direction for Asymmetric Diffusions]\label{thm:identify}
Let $d\mathbf{X}_t = \mathbf{b}(\mathbf{X}_t),dt + \mathbf{G},d\mathbf{W}_t$ be an Itô diffusion with non-conservative drift (i.e., $\mathbf{b}$ is not the gradient of any potential and $\mathbf{J}(\mathbf{X}_t) \ne 0$), and suppose for all $t$ that the marginal density $p_t$ is smooth and full-rank, with invertible $\mathbf{H}$. Let $\mathcal{E}_{\rightarrow}$ and $\mathcal{E}_{\leftarrow}$ denote the average squared error when fitting the empirical drift under the forward and backward order, respectively, using the score-based criterion:
\begin{equation}
\mathcal{E}_{\rightarrow} = \mathbb{E}\left[ \left| \hat{\mathbf{b}} - \mathbf{H}\nabla_{\mathbf{x}}\log \hat{p}_t(\mathbf{X}_t)\right|^2 \right],
\end{equation}
with $\mathcal{E}_{\leftarrow}$ defined analogously for the reversed order. Then the correct temporal order uniquely minimizes the average error:
\begin{equation}
\mathcal{E}_{\text{correct}} < \mathcal{E}_{\text{incorrect}}.
\end{equation}
That is, the direction of time is statistically identifiable using the drift–score discrepancy criterion, provided the process is irreversible.
\end{theorem}

\begin{proof}
\emph{(Sketch of proof).}
For the correct order, empirical increments align with the drift minus the score term as in the time-reversal formula. In the incorrect direction, the empirical drift deviates due to the non-gradient part of $\mathbf{b}$, yielding systematically larger error. Thus, the minimization identifies the true order; for reversible (gradient) drift, both directions are indistinguishable.  
\end{proof}

Once the correct time direction is identified, the data $\{\mathbf{X}_t\}$ forms an ordered sample path of the underlying SDE. The empirical increments $\Delta \mathbf{X}_t = \mathbf{X}_{t+1} - \mathbf{X}_t$ are then governed by the discrete-time Euler--Maruyama dynamics. Parameter estimation thus reduces to a standard identification problem: if $\{\mathbf{X}_t\}$ spans $\mathbb{R}^d$ and $\mathbf{H}$ is positive definite, both $\mathbf{b}$ and $\mathbf{H}$ are uniquely identifiable from the observed trajectory, as stated in the following theorem.

\begin{theorem}[Identifiability of Drift and Diffusion Parameters]\label{thm:learn_params}
Let $\{\mathbf{X}_t\}_{t=0}^{T}$ be a discrete-time trajectory generated by the Euler--Maruyama discretization with step size $\Delta t > 0$ of the $d$-dimensional SDE
\begin{equation}
    d\mathbf{X}_t = \mathbf{b}(\mathbf{X}_t)\,dt + \mathbf{G}\,d\mathbf{W}_t,
\end{equation}
where $\mathbf{b}: \mathbb{R}^d \to \mathbb{R}^d$ is the drift function and $\mathbf{G} \in \mathbb{R}^{d \times m}$, with $\mathbf{H} = \mathbf{G}\mathbf{G}^\top \succ 0$. Assume:
\begin{itemize}
    \item The empirical Gram matrix $\sum_{t=0}^{T-1} \mathbf{X}_t\mathbf{X}_t^\top$ is invertible (i.e., the trajectory $\{\mathbf{X}_t\}$ spans $\mathbb{R}^d$).
    \item The diffusion covariance $\mathbf{H}$ is positive definite.
\end{itemize}
Then, both the drift function $\mathbf{b}$ (up to its linearization in the linear case) and the diffusion matrix $\mathbf{H}$ are identifiable from the observed trajectory.
\end{theorem}

\begin{proof}\emph{(Sketch of proof).}
The Euler--Maruyama transition model yields a Gaussian conditional law with mean $\mathbf{X}_t + \mathbf{b}(\mathbf{X}_t)\Delta t$ and covariance $\mathbf{H}\Delta t$. If two parameter sets produce identical conditionals for all $\mathbf{X}_t$, the means and covariances must agree, implying $\mathbf{b}_1(\mathbf{X}_t) = \mathbf{b}_2(\mathbf{X}_t)$ and $\mathbf{H}_1 = \mathbf{H}_2$. Invertibility of the Gram matrix ensures this holds in $\mathbb{R}^d$.   
\end{proof}

\subsection{Parameter Learning from Unordered Observations}
We estimate SDE parameters via maximum likelihood, assuming linear drift $\mathbf{b}(\mathbf{X}_t) = \mathbf{A}\mathbf{X}_t$ and constant diffusion with covariance $\mathbf{H} \succ 0$ (parameterized directly). For small $\Delta t$, the discrete-time increments follow
\begin{equation}\label{eq:euler_maruyama}
    \Delta \mathbf{X}_i = \mathbf{A}\mathbf{X}_i \Delta t + \boldsymbol{\varepsilon}_i, \qquad
    \boldsymbol{\varepsilon}_i \sim \mathcal{N}(\mathbf{0},\, \mathbf{H} \Delta t).
\end{equation}
The log-likelihood over all increments is
\begin{align}
    \mathcal{L}(\mathbf{A}, \mathbf{H}) &= -\frac{dN}{2} \log(2\pi\Delta t) - \frac{N}{2}\log|\mathbf{H}| \notag \\
    &\quad -\frac{1}{2\Delta t}\sum_{i} \mathbf{R}_i(\mathbf{A})^\top \mathbf{H}^{-1} \mathbf{R}_i(\mathbf{A}),
    \label{eq:log_likelihood}
\end{align}
where
\begin{equation}\label{eq:residual}
    \mathbf{R}_i(\mathbf{A}) = \Delta \mathbf{X}_i - \mathbf{A}\mathbf{X}_i\Delta t.
\end{equation}
The MLEs are given by
\begin{align}
    \hat{\mathbf{A}} &= \frac{1}{\Delta t} \left( \sum_{i} \Delta \mathbf{X}_i \mathbf{X}_i^\top \right)
    \left( \sum_{i} \mathbf{X}_i \mathbf{X}_i^\top \right)^{-1} \label{eq:MLE_A} \\
    \hat{\mathbf{H}} &= \frac{1}{N\Delta t} \sum_{i} 
    \left( \Delta \mathbf{X}_i - \hat{\mathbf{A}} \mathbf{X}_i \Delta t \right)
    \left( \Delta \mathbf{X}_i - \hat{\mathbf{A}} \mathbf{X}_i \Delta t \right)^\top \label{eq:MLE_H}
\end{align}

\begin{algorithm}
    \caption{\Retrace: A score-based iterative method for estimating SDE parameters and reordering trajectories}
    \label{main_algorithm}
    \begin{algorithmic}[1]
        \renewcommand{\algorithmicrequire}{\textbf{Input:}}
        \renewcommand{\algorithmicensure}{\textbf{Output:}}
        \REQUIRE Missing orders $\tilde{\mathbf{X}} = (\tilde{\mathbf{X}}_0, \ldots, \tilde{\mathbf{X}}_{T})$, num iterations $K$, score function $\text{Score}(\mathbf{X}_t) = \nabla_{\mathbf{x}} \log \hat{p}_t(\hat{\mathbf{X}}_t)$
        \ENSURE  Reordered data $\hat{\mathbf{X}}$ and estimated parameters $\hat{\mathbf{A}}, \hat{\mathbf{H}}$
        \STATE Initialize $\hat{\mathbf{X}} \gets \tilde{\mathbf{X}}$
        \FOR {$i = 1$ to $K$}
            \STATE Estimate $(\hat{\mathbf{A}}, \hat{\mathbf{H}})$ using current $\hat{\mathbf{X}}$ (Eq. \eqref{eq:MLE_A}, \eqref{eq:MLE_H})
            \STATE $end \gets T - 1$
            \STATE $converged \gets \text{True}$
            \WHILE {$end \neq 0$}
                \FOR {$t = 0, \ldots, end - 1$}
                    \STATE $s \gets t + 1$
                    \STATE Compute empirical drift: $\hat{\mathbf{b}} = |(\hat{\mathbf{X}}_s - \hat{\mathbf{X}}_t)/\Delta t |$
                    \STATE $\text{Error}_t = \|\hat{\mathbf{b}} - \hat{\mathbf{H}} \cdot \text{Score}(\mathbf{X}_t)\|^2$\label{Error_Xt}
                    \STATE $\text{Error}_s = \|\hat{\mathbf{b}} - \hat{\mathbf{H}} \cdot \text{Score}(\mathbf{X}_s)\|^2$\label{Error_Xs}
                    \IF {$\text{Error}_t < \text{Error}_s$}
                        \STATE Swap $\hat{\mathbf{X}}_t$ and $\hat{\mathbf{X}}_s$
                        \STATE $converged \gets \text{False}$
                    \ENDIF
                \ENDFOR
                \STATE $end \gets end - 1$
            \ENDWHILE
            \STATE Check \textbf{if} {$converged$} \textbf{then} \textbf{break}
        \ENDFOR
        \RETURN $\hat{\mathbf{X}}, \hat{\mathbf{A}}, \hat{\mathbf{H}}$
    \end{algorithmic}
\end{algorithm}

After estimating $\hat{\mathbf{A}}$ and $\hat{\mathbf{H}}$, we iteratively reorder $\{\mathbf{X}_t\}$ to maximize likelihood under the SDE dynamics. For each adjacent pair $(\mathbf{X}_t, \mathbf{X}_s)$, we compute the empirical drift $\hat{\mathbf{b}} = |(\mathbf{X}_s - \mathbf{X}_t)/\Delta t|$ and the empirical score $\nabla_{\mathbf{x}} \log \hat{p}_t(\mathbf{X}_t)$, estimated as $-\hat{\boldsymbol{\Sigma}}_t^{-1}(\mathbf{X}_t - \hat{\boldsymbol{\mu}}_t)$. The drift–score discrepancy error is given by
\begin{equation}\label{eq:score_error}
    \mathrm{Error}(\mathbf{X}_t) = \left\| \hat{\mathbf{b}} - \hat{\mathbf{H}} \cdot \nabla_{\mathbf{x}}\log \hat{p}_t(\mathbf{X}_t) \right\|^2.
\end{equation}


At each iteration, we compute drift--score errors for both possible orderings of each adjacent pair and swap states to minimize the total error, repeating this process until convergence. This score-based iterative sorting produces a trajectory whose temporal direction is maximally consistent with the estimated SDE dynamics. We refer to our algorithm as \Retrace (REcover TRAjectories and learn sde parameters for Counterfactual Estimates). The full procedure is detailed in Algorithm~\ref{main_algorithm}, with theoretical recovery guarantees given in below Theorem~\ref{thm:all}.

\begin{theorem}\label{thm:all}
Let $\bigl\{\mathbf{X}_{t}\bigr\}_{t=0}^{T-1} \subset \mathbb{R}^d$ be a time series generated from a time-homogeneous linear additive noise stochastic differential equation (SDE) as in Theorem~\ref{thm:identify}. Suppose the observed trajectories $\bigl\{\hat{\mathbf{X}}_{j}\bigr\}_{j=0}^{T-1}$ are a randomly permuted version of the original ordered data $\bigl\{\mathbf{X}_{t}\bigr\}_{t=0}^{T-1} \subset \mathbb{R}^d$, such that temporal order is lost. We assume that the score function $\nabla_{\mathbf{x}}\log p_t(\mathbf{X}_t)$ is well-defined and can be consistently estimated by a model $\text{Score}(\mathbf{X}_t)$ for each $\mathbf{X}_t$.
Then, the correct temporal order $\mathbf{X}_t \prec \mathbf{X}_s$ can be recovered by minimizing the squared errors as defined in lines \ref{Error_Xt} and \ref{Error_Xs} of Algorithm \ref{main_algorithm}.
\end{theorem}

\begin{proof}
\emph{(Sketch of proof).}
Apply Theorems~\ref{thm:identify} and \ref{thm:learn_params}, and use the correctness of the sorting procedure in Algorithm~\ref{main_algorithm} given pairwise orders.
\end{proof}

\section{Experiments}

We validate our method for jointly recovering temporal order and estimating SDE parameters, enabling counterfactual inference and longitudinal treatment effect (LTE) estimation. First, we benchmark trajectory reconstruction and parameter estimation accuracy on datasets simulated from irreversible SDEs, assuming existence and uniqueness of solutions as described in the methods section. Second, we evaluate treatment effect prediction on pharmacological datasets under varying noise levels.

\paragraph{Baseline Methods}
We compare our method to existing approaches for temporal order reconstruction and SDE parameter estimation, as mentioned in Related Work. The baselines are:
\begin{itemize}
    \item \textbf{Minimum Spanning Tree (MST):} Constructs a graph with observations as nodes and edge weights given by pairwise Euclidean distances~\cite{trapnell2014dynamics}; the MST is traversed to yield a temporal ordering.
    \item \textbf{Diffusion Pseudotime (DPT):} Infers pseudotemporal order via diffusion maps and random walks~\cite{Haghverdi2016}. DPT does not exploit SDE dynamics and serves as a baseline for comparison.
\end{itemize}
Given the reconstructed orderings, we estimate SDE parameters via maximum likelihood estimation (MLE). For our method, we additionally benchmark ordinary least squares (OLS) and expectation-maximization (EM) for parameter learning. We use algorithm~\ref{main_algorithm} with 10 iterations in default.

\paragraph{Evaluation Metrics}
We evaluate methods using temporal order reconstruction accuracy, parameter estimation error, and runtime complexity:
\begin{itemize}
\item \textit{Ordering Accuracy:} Let $\mathbf{X}$ denote the original (ordered) data and $\hat{\mathbf{X}}$ the reordered data. We define the average accuracy as
\[
    \text{Acc.} = \frac{1}{N} \sum_{j=1}^N \frac{1}{T} \sum_{t=1}^{T} \mathbb{I}\{\text{Order}(\mathbf{X}_j)_t = \text{Order}(\hat{\mathbf{X}}_j)_t\}
\]
where $N$ is the number of trajectories, $T$ the number of time steps, and $\mathbb{I}\{\cdot\}$ the indicator function.

\item \textit{Parameter Estimation Error:} Let $(\hat{\mathbf{A}}, \hat{\mathbf{H}})$ be the estimated drift and diffusion parameters, with ground-truth $(\mathbf{A}, \mathbf{H})$. We report the mean absolute error (MAE) between the ground-truth and estimated parameters:
\begin{align*}
\text{MAE} &= \frac{1}{d^2} \sum_{i=1}^{d} \sum_{j=1}^{d} \left| \mathbf{A}_{ij} - \hat{\mathbf{A}}_{ij} \right|,
\end{align*}

where $d$ is the system dimension.

\item \textit{Runtime Complexity:} We report average iteration runtime (in seconds) required by each method for both temporal order reconstruction and parameter estimation.
\end{itemize}

\subsection{Irreversible SDE Datasets}
We generate irreversible time-homogeneous linear additive noise SDEs (Eq.~\ref{time_homo_linear_additive_noise_SDE}) of dimension $d=50$ with random parameter $A$ and $G$ satisfying probability current $\mathbf{J}(\mathbf{x}) \ne 0$. We then use Euler-Maruyama discretization
$\mathbf{X}_{i+1}^{(j)} = \mathbf{X}_i^{(j)} + \mathbf{A} \mathbf{X}_i^{(j)} \Delta t + \mathbf{G} d\mathbf{W}_i$
with $\Delta t=0.01$ and sample 2000 trajectories, each with 250 timesteps. From this, we then randomize the data along the time-step dimension to obtain the missing-order dataset $\tilde{\mathbf{X}}$. 
This data will then be the input for all methods to reorder and learn the underlying parameters $\mathbf{A}, \mathbf{H}$.

\paragraph{Results}
Table~\ref{tab:results_summary} shows that ReTrace, our temporal sorting method, outperforms the baselines MST and DPT on all metrics. For temporal order recovery, ReTrace achieves substantially higher accuracy---with MLE and EM parameter estimation yielding nearly perfect results ($99.1 \pm 2.6$\% and $98.3 \pm 3.5$\%, respectively), compared to MST ($22.1 \pm 15.1$\%) and DPT ($4.8 \pm 7.0$\%). While all three estimation strategies (MLE, OLS, EM) within ReTrace are effective, MLE and EM further reduce the mean absolute error (MAE) in recovering the drift ($A$) and diffusion ($H$) parameters. In contrast, the baselines are only partially successful and exhibit much larger MAEs.

The performance of ReTrace is attributed to its principled score-based order recovery, which remains robust as the number of time steps increases. Moreover, high accuracy is typically achieved within 1--2 epochs, suggesting that early stopping can further improve computational efficiency.

\begin{table}[ht]
\centering
\setlength{\tabcolsep}{0pt} 
\begin{tabular*}{\columnwidth}{l@{\extracolsep{\fill}}ccc}
\toprule
\textbf{Method} & \textbf{Accuracy} & \textbf{MAE-A} & \textbf{MAE-H} \\
\midrule
ReTrace-MLE & $99.1 \pm 2.6$ & $0.05 \pm 0.03$ & $0.1 \pm 0.07$ \\
ReTrace-OLS & $93.6 \pm 12.8$ & $3.9 \pm 3.5$ & $5.0 \pm 9.7$ \\
ReTrace-EM & $98.3 \pm 3.5$ & $0.1 \pm 4.2$ & $11.4 \pm 10.1$ \\
MST-MLE & $22.1 \pm 15.1$ & $3.1 \pm 3.9$ & $8.5 \pm 10.5$ \\
DPT-MLE & $4.8 \pm 7.0$   & $5.0 \pm 4.2$ & $11.4 \pm 11.6$ \\
\bottomrule
\end{tabular*}
\caption{Comparison under the base setting between our method (Our + MLE) and baseline methods with regard to A/H MAEs, Sorting Accuracy, and Iteration Runtime. Higher means better for Sorting Accuracy, whereas for the others, it is the opposite. Summary of performance metrics: mean $\pm$ standard deviation (rounded to 1 decimal place). All values are absolute.}
\label{tab:results_summary}
\end{table}

\subsubsection{With Observation Noises}
The observation noise which is due to noisy measurements further corrupt the data and affect parameter learning \cite{sun2025robustparameterestimationdynamical, NIPS2012_5c936263}. The observation noise model is defined as:
\begin{equation}\label{noisy_measurement_sigma}
    \mathbf{Y}_t = \mathbf{X}_t + \boldsymbol{\epsilon}_t, \hspace{0.15cm}
    \boldsymbol{\epsilon}_t \sim \mathcal{N}(\mathbf{0}, \hspace{0.05cm} \mathbf{R}),
\end{equation}
where $\mathbf{R} = \sigma_\epsilon^2\mathbf{I}$.
We compare the methods under this challenging scenarios. Note that the measurement noise here is different than the white noise driving our SDE. We generate data that are corrupted by zero mean Gaussian noise at each time step. Because both the state and the noise are Gaussian, the observed data at a single time-slice are still Gaussian, so we can still use our proposed Algorithm \ref{main_algorithm}, with some small modifications:
\begin{align*}
    \hat{\mathbf{S}}_t &= \frac{1}{N-1}\sum_{j=1}^{N}(\mathbf{Y}^{(j)}_t - \hat{\boldsymbol{\mu}}_t)(\mathbf{Y}^{(j)}_t - \hat{\boldsymbol{\mu}}_t)^{\top},\\
    \hat{\boldsymbol{\Sigma}}_t &= \hat{\mathbf{S}}_t - \mathbf{R},
\end{align*}
where $\hat{\mathbf{S}}_t$ is the sample covariance of $\mathbf{Y}_t$ instead of $\mathbf{X}_t$, we use the same mean $\hat{\boldsymbol{\mu}}_t$ since the measurement noise $\boldsymbol{\epsilon}_t$ is zero-mean.

\begin{figure}[ht]
  \centering
  \begin{subfigure}[t]{\columnwidth}
    \centering
    \includegraphics[width=\linewidth]{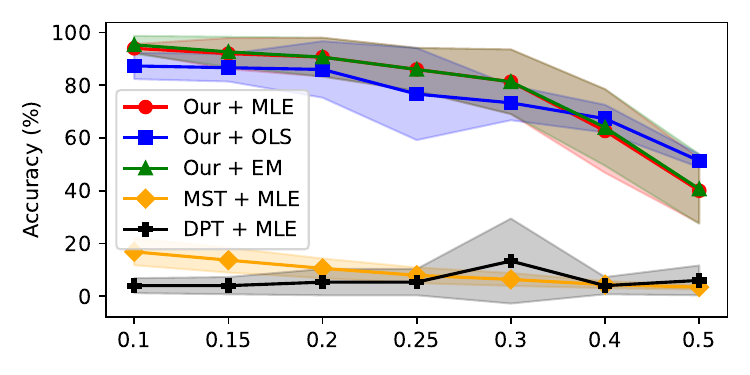}
    \caption{Reordering Accuracy}
    \label{fig:figure3}
  \end{subfigure}

  \begin{subfigure}[t]{\columnwidth}
    \centering
    \includegraphics[width=\linewidth]{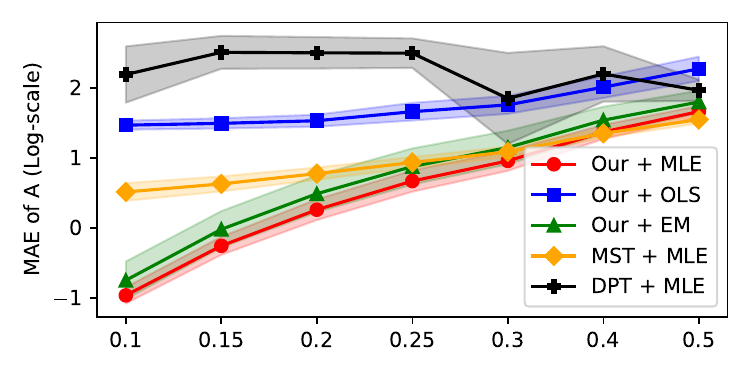}
    \caption{Mean Absolute Error of drift $\mathbf{A}$ }
    \label{fig:figure1}
  \end{subfigure}

  \caption{Performance comparisons between ReTrace ~and baseline methods on (a) Reordering accuracy, and (b) MAEs of drift parameter A under increasing observation noises. The x-axis is the noise percentage.}
  \label{noisy_figures}
\end{figure}

\subsubsection{Results Analysis}
We present the results of our experiments here. Then we discuss how it showcases our method's robustness and effectiveness compared to the baselines.

For this noisy measurements setting, we also use some fixed parameters such as $d = 10$, num\_trajectories $= 2000$, num\_steps $= 50$. The parameter noisy\_measurements\_sigma ranges from $0.1$ to $0.5$, which is $\sigma$ from equation \ref{noisy_measurement_sigma}.

Figure 2 illustrates the impact of increasing observation noise on both temporal order reconstruction accuracy and drift parameter estimation error (MAE of $\mathbf{A}$). As the noise level rises gradually from 0.1 to 0.5, all methods experience degradation in performance. However, our proposed method (ReTrace) consistently outperforms baseline approaches, maintaining substantially higher reordering accuracy and lower parameter estimation errors across all noise levels. This indicates that our score-based sorting procedure is far less sensitive to measurement corruption compared to graph-based and diffusion-pseudotime methods.

\subsection{Pharmacological Data Experiments}
\label{sec:sde_time_direction_benchmark}

Synthetic datasets are commonly used in benchmarking longitudinal treatment effect (LTE) methods, since counterfactual outcomes are unobservable in real datasets~\cite{lim2018forecasting, bica2020estimating}. Following this standard, we construct a testbed using synthetic trajectories generated from stochastic pharmacological models~\cite{kacprzyk2024ode}.

\paragraph{Synthetic Data Generation.}
We generate data from different stochastic pharmacological models $\mathcal{F}$ by varying the noise, static covariates, and parametric parameter distributions instantiated from the \textbf{stochastic tumor growth PKPD SDE}:
\begin{align}\label{eq:pkpd}
    dX_t = & \left[\rho\log\left(\frac{K}{X_t}\right) - \beta_c C(t) \right. \nonumber \\
    & \left. - \left(\alpha_r d(t) + \beta_r d(t)^2\right)\right] X_t\,dt + \sigma_{\mathrm{tumor}}\,dW_t
\end{align}
with continuous chemotherapy $C(t)$, binary radiotherapy $d(t)$, and subject-specific parameters ($\rho, K, \beta_c, \alpha_r, \beta_r, \sigma_{\mathrm{tumor}}$).

\paragraph{Discrete-Time Transition Model.}
The tumor growth dynamics are simulated by discretizing the univariate PKPD SDE via the Euler--Maruyama scheme. At each time step, the tumor volume $X_t$ evolves according to
\[
    X_{t+1} = X_t + b(X_t, C_t, d_t)\,\Delta t + \sigma_{\mathrm{tumor}}\,\Delta W_t,
\]
where $\Delta W_t \sim \mathcal{N}(0, \Delta t)$, $C_t$ is the administered chemotherapy dosage (continuous), and $d_t$ is the radiotherapy indicator (binary). The drift term is specified as
\[
    b(X_t, C_t, d_t) = \left(\rho\log\frac{K}{X_t} - \beta_c C_t - \left(\alpha_r d_t + \beta_r d_t^2\right)\right) X_t
\]
with all parameters patient-specific. The diffusion coefficient $\sigma_{\mathrm{tumor}}$ encapsulates intrinsic biological noise as well as additional variability due to unmodeled patient-specific effects and parameter heterogeneity. This leads to the following conditional transition distribution:
\[
    X_{t+1} \mid X_t, C_t, d_t \sim \mathcal{N}\left(X_t + b(X_t, C_t, d_t)\,\Delta t,\sigma_{\mathrm{tumor}}^2\,\Delta t\right)
\]
which jointly models the deterministic effects of clinical interventions and stochastic, patient-dependent evolution of tumor size.

We simulate $N = 5000$ patient-specific tumor growth trajectories, see Fig.~\ref{fig:tumor_traj}, governed by a stochastic pharmacokinetic/pharmacodynamic (PKPD) SDE model in Eq.~\eqref{eq:pkpd}
where $X_t$ is the tumor volume at time $t$, $C(t)\in [0, \mathtt{max\_chemo}]$ is the continuous chemotherapy dose, and $d(t)\in \{0,1\}$ denotes binary radiotherapy. Subject-specific parameters $(\rho, K, \beta_c, \alpha_r, \beta_r, \sigma_{\mathrm{tumor}})$ are sampled from log-normal distributions with 5\% between-subject variability. The initial tumor diameter is drawn uniformly from $[13, 15]$\,mm and converted to volume. Treatment policies depend on current tumor volume via a sigmoid function with confounding strength $\gamma=2.0$. Trajectories are simulated using the Euler--Maruyama method with $T=15$, $n=60$ time steps, and additive observational noise $\sigma_{\mathrm{obs}}=0.01$. Both fully-treated and untreated trajectories are generated to evaluate robustness across interventional conditions.
\begin{figure}[!ht]
    \centering
    \includegraphics[width=0.95\linewidth]{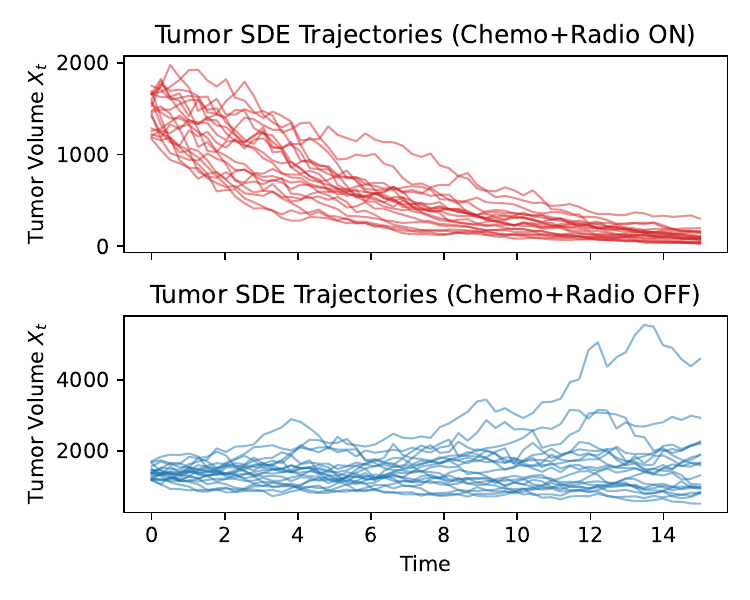}
    \caption{Sample tumor growth trajectories generated from the stochastic PKPD SDE model. (\textbf{Top}) Treated group with chemotherapy and radiotherapy continuously administered. (\textbf{Bottom}) Untreated group with no interventions. Each curve represents an independent patient sampled from the SDE in Eq.~\ref{eq:pkpd}.}
    \label{fig:tumor_traj}
\end{figure}

\paragraph{Trajectory Construction.}
For each synthetic subject, we simulate a discrete trajectory of $T$ steps, sampling initial conditions, covariates, and treatment sequences. Both factual and counterfactual outcome paths are simulated for each subject by altering the intervention sequence.

\begin{figure}[ht]
  \centering
  \includegraphics[width=\columnwidth]{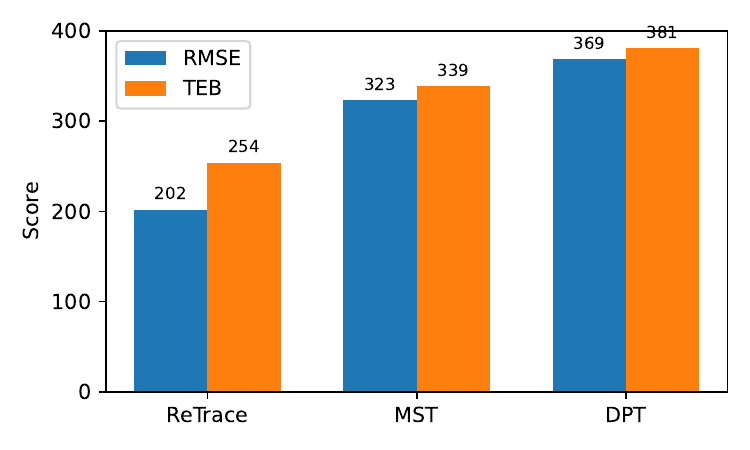}
  \caption{Root Mean Squared Error (RMSE) and Treatment Effect Bias (TEB), computed using~\eqref{eq:rmse} and~\eqref{eq:teb}.}
  \label{fig:rmse-teb}
\end{figure}

\paragraph{Direction-of-Time Ambiguity and Evaluation.}
To rigorously test the ability to recover temporal order,
\begin{itemize}
    \item For a subset of samples, the true time order of states $(X_1 \to X_2 \to \ldots \to X_T)$ is withheld, and all possible permutations are considered.
    \item The task is to infer, via SDE-based likelihoods, which ordering best fits the observed data under the learned dynamics.
\end{itemize}

\paragraph{Assumptions for SDE-based Treatment Effect Estimation.}

We model patient-level tumor dynamics under treatment and control as stochastic differential equations (SDEs) of the form
\begin{equation}
    dX_t = b(X_t, a)\,dt + G\,dW_t,
\end{equation}
where $X_t \in \mathbb{R}$ denotes the tumor volume at time $t$, $a \in \{0,1\}$ encodes treatment assignment ($a=1$ for treated, $a=0$ for control), $b(\cdot, a)$ is a Lipschitz continuous drift function capturing the deterministic effect of treatment and tumor progression, and $W_t$ is a standard Wiener process. To realistically model between-subject variability (BSV) arising from intrinsic heterogeneity and unmeasured confounders, we incorporate additive process noise via the Wiener process, with diffusion coefficient $G$ (potentially patient-specific) encoding both measurement noise and subject-level random effects.

\paragraph{Counterfactual Treatment Effect Estimation via SDEs.}

Given observed data under treatment ($a=1$) or control ($a=0$), we infer SDE parameters for each regime using maximum likelihood or score-based estimation, accommodating unordered or irregularly sampled patient trajectories. For counterfactual inference, we simulate potential outcomes by numerically solving the SDE for both $a=1$ (treatment) and $a=0$ (control) for each patient, conditioned on their baseline covariates and initial tumor volume. The individualized treatment effect (ITE) for patient $i$ is estimated as the difference in expected tumor volume at time $t^*$ under treatment versus control:
\begin{equation}
    \text{ITE}_i = \mathbb{E}[X_{t^*} \mid a=1] - \mathbb{E}[X_{t^*} \mid a=0].
\end{equation}
The average treatment effect (ATE) is computed by averaging the ITEs across the cohort.

\paragraph{Evaluation Metrics.}
We assess model performance on two axes:
\begin{itemize}
    \item \textbf{Treatment Effect Bias:} The mean bias in ITE/ATE is quantified as the average difference between the inferred treatment effects (using the estimated temporal order and SDE parameters) and the ground-truth effects (using the true order and known parameters). Formally, for $N$ patients,
    \begin{equation}\label{eq:teb}
        \text{Treatment Effect Bias} = \frac{1}{N} \sum_{i=1}^N \left[ \widehat{\text{ITE}}_i - \text{ITE}_i^{\mathrm{(true)}} \right].
    \end{equation}
    \item \textbf{Counterfactual Predictive Error:} The accuracy of counterfactual predictions is measured by the root mean squared error (RMSE) between the predicted and ground-truth counterfactual tumor volumes at $t^*$:
    \begin{equation}\label{eq:rmse}
    \begin{split}
        &\text{Counterfactual RMSE} = \\
        &\quad \sqrt{ \frac{1}{N} \sum_{i=1}^N \left( \widehat{X}_{i, t^*}^{\mathrm{(cf)}} - X_{i, t^*}^{\mathrm{(cf, true)}} \right)^2 }.
    \end{split}
    \end{equation}
\end{itemize}

This framework enables rigorous quantification of both the causal effect of treatment and the reliability of counterfactual predictions under model- and data-induced uncertainty, leveraging SDE-based dynamical models and accounting for realistic subject-level heterogeneity via process noise.

\section{Conclusion}
We propose ReTrace, a score-based framework for recovering temporal order and estimating SDE parameters from unordered data. ReTrace reconstructs trajectories via pairwise score minimization and enables consistent parameter inference for counterfactual and treatment effect estimation. Experiments on synthetic and real data validate its effectiveness in incomplete or privacy-sensitive temporal settings, which opens up many research directions or techniques in such extreme scenarios.

\section{Acknowledgments}
The authors wish to thank the anonymous AAAI reviewers for their insightful feedback and constructive comments, which significantly improved the quality of this paper. We also acknowledge the support of Deakin Applied Artificial Intelligence Initiative (A\textsuperscript{2}I\textsuperscript{2}).

I also wish to thank my family and friends, especially Chloe, for their constant support and belief in me.

\bibliography{aaai2026}
\end{document}